\documentclass[12pt]{article}
\usepackage{times}
\usepackage{amssymb}
\usepackage{amsmath}
\usepackage{graphicx}
\usepackage{chicagor}

\newtheorem{theorem}{Theorem}[section]
\newtheorem{proposition}[theorem]{Proposition}

\newtheorem{exam}[theorem]{Example}
\newcommand{\bbox}{\vrule height7pt width4pt depth1pt}
\newenvironment{example}{\begin{exam} \rm }{\hfill $\bbox$ \end{exam}}
\newtheorem{rema}[theorem]{Remark}

\newtheorem{obs}[theorem]{Observation}

\newtheorem{defin}[theorem]{Definition}
\newenvironment{definition}{\begin{defin} \rm }{\hfill $\bbox$
\end{defin}}
\newenvironment{proof}{\noindent {\bf Proof.}}{\hfill $\bbox$\\}
\newcommand{\Oo}{\Omega}
\newcommand{\oo}{\omega}
\newcommand{\ffi}{\varphi}

\newcommand{\vDasho}{\vDash^{\ou}}
\newcommand{\ou}{\mathit{out}}
\newcommand{\vDashi}{\vDash^{\inner}}
\newcommand{\inner}{\mathit{in}}
\newcommand{\vDashos}{\vDash^{\ous}}
\newcommand{\ous}{\mathit{out,ai}}
\newcommand{\vDashis}{\vDash^{\inners}}
\newcommand{\inners}{\mathit{in,ai}}

\newcommand{\M}{\mathcal{M}}
\newcommand{\Mis}{\mathcal{M}^{\mathit{ai}}}
\newcommand{\N}{\mathcal{N}}
\newcommand{\Mcpa}{\mathcal{M}^{\mathit{cpa}}}
\newcommand{\Miscpa}{\mathcal{M}^{\mathit{cpa,ai}}}

\newcommand{\F}{\mathcal{F}}
\newcommand{\B}{{\mathcal B}}

\newcommand{\pr}{{\mathit{pr}}}
\newcommand{\commentout}[1]{}

\renewcommand{\iff}{\Leftrightarrow}
\newcommand{\true}{\mathit{true}}

\newcommand{\EB}{\mathit{EB}}
\newcommand{\CB}{\mathit{CB}}

\title{A logic for reasoning about ambiguity}

\author{{Joseph Y. Halpern}\\
Computer Science Dept.\\
Cornell University\\
Ithaca, NY\\
E-mail: halpern@cs.cornell.edu\\
\and
Willemien Kets\\
Kellogg School of Management\\
Northwestern University\\
Evanston, IL \\
E-mail: w-kets@kellogg.northwestern.edu
}

\begin{document}
\maketitle

\begin{abstract}
Standard models of multi-agent modal logic do not capture the fact that
information is often \emph{ambiguous}, and may be interpreted in different
ways by different agents. We propose a framework that can model this,
and consider different semantics that capture different assumptions
about the agents' beliefs regarding whether or not there is ambiguity.
We examine the expressive power of logics of ambiguity compared to
logics that cannot model ambiguity, with respect to the different
semantics that we propose.
\end{abstract}

\section{Introduction} \label{sec:intro}

In the study of multi-agent modal logics, it is implicitly
assumed that all agents interpret all
formulas the same way.  While they may have different beliefs regarding
whether a formula $\ffi$ is true, they agree on what $\ffi$ means. Formally, this is captured by the fact that the truth of $\ffi$ does not depend on the agent.

Of course, in the real world, there is \emph{ambiguity}; different
agents may interpret the same utterance in different ways.
For example, consider a public announcement $p$. Each player $i$ may
interpret $p$ as corresponding to some event $E_i$, where $E_i$ may be
different from $E_j$ if $i \ne j$.  This seems natural: even if
people have a common background, they may still disagree on how to
interpret certain phenomena or new
information. Someone may interpret a smile as just a sign of
friendliness; someone else may interpret it as a ``false'' smile,
concealing contempt; yet another person may interpret it as a sign of
sexual interest.

To model this formally, we can use a straightforward approach already
used in \cite{Hal43,GroveH2}: formulas are interpreted
relative to a player.  But once we allow such ambiguity, further
subtleties arise.  Returning to the announcement $p$, not only can it be
interpreted differently by different players, it may not even occur
to the players that
others may interpret the announcement in a different way.
Thus, for example, $i$ may believe that $E_i$ is
common knowledge.
The assumption that each player believes that her interpretation is how
everyone interprets the announcement is but one assumption we can make
about ambiguity.
It is also possible that player $i$ may be aware that
there is more than one interpretation of $p$, but believes that player
$j$ is aware of only one interpretation. For example, think of a
politician making an ambiguous statement which he realizes that
different constituencies will interpret differently, but will not
realize that there are other possible interpretations.
In this paper, we investigate a number of different semantics of
ambiguity that correspond to some standard assumptions that people make
with regard to ambiguous statements, and investigate their
relationship.

Our interest in
ambiguity was originally motivated by a seminal result in game theory: Aumann's
\citeyear{Au} theorem showing that players cannot ``agree to
disagree.''  More precisely, this theorem says that agents with a common
prior on a state space cannot have common knowledge that they have
different posteriors.
This result has been viewed as paradoxical in the economics literature.
Trade in a stock market seems to require common knowledge of
disagreement (about the value of the stock being traded), yet we clearly
observe a great deal of trading.
%
One well known explanation for the disagreement is that we do not in
fact have common priors: agents start out with different beliefs.
In a companion paper \cite{HK13}, we provide a different explanation, in
terms of ambiguity.
 It is easy to show that we can agree to disagree when
there is ambiguity, even if there is a common prior.


\commentout{
We then show that these two explanations of the possibility of agreeing
to disagree are closely related, but not identical. We can convert an
explanation in
terms of ambiguity to an explanation in terms of lack of common
priors.\footnote{More precisely, we can convert a structure with ambiguity
and a common prior to an equivalent model---equivalent in the sense that
the same formulas are true---where there is no ambiguity but no common
prior.}
Importantly, however, the converse does not hold; there are models in
which players have a common interpretation that cannot in general be
converted into an equivalent model with ambiguity and a common prior. In
other words, using heterogeneous priors may be too permissive if we are
interested in modeling a situation where differences in beliefs are due
to differences in interpretation.
}

Although our work is motivated by applications in economics,
ambiguity has long been a concern in philosophy, linguistics, and
natural language processing.  For example, there has been
a great deal of work on word-sense
disambiguation (i.e., trying to decide from context which of the
multiple meanings of a word are intended);
see Hirst \citeyear{Hirst_1988} for a seminal contribution, and Navigli
\citeyear{Navigli09} for a recent survey.
However, there does not seem to be much work on incorporating ambiguity
into a logic.
Apart from the literature on the logic of context and on
underspecification (see Van Deemter and Peters
\citeyear{VanDeemterPeters_1996}),
the only papers that we are aware of that does this are
ones by Monz \citeyear{Monz99} and Kuijer \citeyear{Kui13}.  Monz allows
for statements that have
multiple interpretations, just as we do.
But rather than
incorporating the ambiguity directly into the logic, he considers updates
by ambiguous statements.  

Kuijer models the fact that ambiguous statements can have multiple
meanings by using a nondeterministic propositional logic, which, roughly
speaking allows him to consider all the meanings simultaneously.  He
then defines a notion of implication such that an ambiguous statement
$A$ entails another ambiguous statement $B$ if and only if every possible
interpretations of $A$ entails every possible interpretation of $B$.
This idea of considering all possible interpretations of an ambiguous
statement actually has a long tradition in the philosophy literature.
For example, Lewis \citeyear{Lewis82} considers
assigning truth values to an ambiguous formula $\phi$ by considering all
possible disambiguations of $\phi$.  This leads to a semantics where a
formula can, for example, have the truth value $\{\mathbf{true},
\mathbf{false}\}$.  Lewis views this as a potential justification for
\emph{relevance logic} (a logic where a formula can be true, false,
both, or neither; cf.~\cite{ReBr}).  Our approach is somewhat different.  We
assume that each agent uses only one interpretation of a given ambiguous
formula $\phi$, but an agent may consider it possible that another agent
interprets $\phi$ differently.  In our applications, this seems to be
the most appropriate way to dealing with ambiguity (especially when it
comes to considering the strategic implications of ambiguity).

There are also connections between ambiguity
and vagueness.  Although the two notions are different---a term is
\emph{vague} if it is not clear what its meaning is, and
is \emph{ambiguous} if it can have multiple meanings, Halpern
\citeyear{Hal43} also used agent-dependent interpretations in his
model of vagueness, although the issues that arose were quite different
from those that concern us here.

Given the widespread interest in ambiguity, in this paper we focus on
the logic of ambiguity.  We introduce the logic in
Section~\ref{sec:model}.  The rest of the paper is devoted to arguing
that, in some sense, ambiguity is not necessary.  In
Section~\ref{sec:equivalence}, we show that a formula is satisfiable
in a structure with ambiguity (i.e., one where different agents
interpret formulas differently) if and only if it is satisfiable in a structure
without ambiguity.
Then in Section~\ref{sec:genlanguage}, we show that, by extending the
language so that we can talk explicitly about how agents interpret
formulas, we do not need structures with ambiguity.
Despite that, we argue in Section~\ref{sec:concl} that we it is useful to
be able to model ambiguity directly, rather than indirectly.

\commentout{
The rest of this paper is organized as follows.
 Section \ref{sec:model} introduces the logic that we consider. Section
\ref{sec:examples} investigates the implications of the common-prior
assumption when there is ambiguity. Section \ref{sec:equivalence}
studies the tradeoff between heterogeneous priors and ambiguity, and
Section \ref{sec:concl} concludes.
}

\section{Syntax and Semantics} \label{sec:model}

\subsection{Syntax}\label{sec:syntax}

We want a logic where players use a fixed common language,
but each player may interpret formulas in the language differently.
Although we do not need probability for the points we want to make in
this paper, for the applications that we have in mind it is
also important for the agents to be able to reason about their
probalistic beliefs.  Thus, we take as our base logic a propositional logic
for reasoning about probability.

The syntax of the logic is straightforward
(and is, indeed, essentially the syntax already used in papers going
back to Fagin and Halpern \citeyear{FH3}).
There is a finite, nonempty set $N = \{1, \ldots, n\}$ of players, and
a countable, nonempty set $\Phi$ of primitive propositions. Let
$\mathcal{L}_n^C(\Phi)$ be the set of formulas
that can be constructed starting from $\Phi$, and closing off under
conjunction, negation, the modal operators
$\{\CB_G\}_{G \subseteq N, G \neq \emptyset}$, and the formation
of probability formulas.
(We omit the $\Phi$ if it is irrelevant or clear from context.)
  Probability formulas are
constructed as follows. If
$\ffi_1, \ldots, \ffi_k$ are formulas, and $a_1, \ldots, a_k, b
\in \mathbb{Q}$, then for $i \in N$,
\[
a_1 \pr_i(\ffi_1) + \ldots + a_k \pr_i(\ffi_k) \geq b
\]
is a probability formula,
where $\pr_i(\ffi)$ denotes the probability that player $i$
assigns to a formula $\ffi$.
Note that this syntax
allows for nested probability formulas. We use the abbreviation
$B_i \ffi$ for $\pr_i (\ffi) = 1$, $\EB_G^1 \ffi$ for $\wedge_{i \in G} B_i
\ffi$, and
$\EB^{m+1}_G\ffi$ for $\EB^m_G \EB^1_G\ffi$ for $m = 1,2 \ldots$.
Finally, we take $\true$ to be the abbreviation for a fixed tautology
such as $p \vee \neg p$.

\subsection{Epistemic probability structures}

There are standard approaches for interpreting this language \cite{FH3},
but they all assume that there is no ambiguity, that
is, that all players interpret the primitive propositions the same way.
To allow for different interpretations, we use an approach used earlier
\cite{Hal43,GroveH2}: formulas are interpreted
relative to a player.

An \emph{(epistemic probability) structure} (\emph{over $\Phi$})
has the form
\[
M = (\Oo, (\Pi_j)_{j \in N}, ({\cal P}_j)_{j \in N}, (\pi_j)_{j \in N}),
\]
where $\Oo$ is the state space, and for each $i \in N$, $\Pi_i$
is a partition of $\Oo$, ${\cal P}_i$ is a function that
assigns to each $\oo \in \Oo$ a probability space ${\cal
P}_i(\oo) = (\Oo_{i,\oo}, \F_{i,\oo}, \mu_{i,\oo})$,
and $\pi_i$ is an interpretation that associates with each
state a truth assignment to the primitive propositions in $\Phi$. That
is, $\pi_i(\oo)(p) \in \{\mathbf{true}, \mathbf{false}\}$ for
all $\oo$ and each primitive proposition $p$.
Intuitively, $\pi_i$ describes player $i$'s interpretation of the
primitive propositions. Standard models use only a single
interpretation $\pi$; this is equivalent in our framework to assuming
that $\pi_1 = \cdots = \pi_n$. We call a structure where $\pi_1 = \cdots = \pi_n$ a
\emph{common-interpretation structure};
we call a structure where $\pi_i \ne \pi_j$ for some agents $i$ and $j$
a structure \emph{with ambiguity}.
Denote by $[[p]]_i$ the set of states where $i$ assigns the value
$\mathbf{true}$ to $p$.
The partitions $\Pi_i$ are called \emph{information partitions}.
While it is more standard in the philosophy and computer science
literature to use models where there is a binary relation ${\cal K}_i$
on $\Oo$ for each agent $i$ that describes $i$'s accessibility relation
on states,
we follow the common approach in economics
of working with information partitions here, as that makes it
particularly easy to define a player's probabilistic beliefs. Assuming
information partitions corresponds to the case that ${\cal K}_i$ is an
equivalence relation (and thus defines a partition).  The intuition is
that a cell in the partition $\Pi_i$ is defined by some information that
$i$ received, such as signals or observations of the world.
Intuitively, agent $i$
receives the same information at each state in a cell of $\Pi_i$.
Let $\Pi_i(\oo)$ denote the cell of the partition $\Pi_i$ containing $\oo$.
Finally, the probability space ${\cal
P}_i(\oo) = (\Oo_{i,\oo}, \F_{i,\oo}, \mu_{i,\oo})$ describes the
beliefs of player $i$ at state $\oo$, with $\mu_{i,\oo}$ a probability
measure defined on the subspace $\Oo_{i,\oo}$ of the state space $\Oo$.
The $\sigma$-algebra $\F_{i, \oo}$ consists of the subsets of
$\Oo_{i,\oo}$ to which $\mu_{i,\oo}$ can assign a probability.  (If
$\Oo_{i,\oo}$ is finite, we typically take $\F_{i,\oo} =
2^{\Oo_{i,\oo}}$, the set
of all subsets of $\Oo_{i,\oo}$.)
The interpretation is that $\mu_{i,\oo}(E)$ is the probability that
$i$ assigns to event $E \in \F_{i,\oo}$ in state $\oo$.

Throughout this paper, we make the following assumptions regarding the
probability assignments ${\cal P}_i$, $i \in N$:
\begin{description}
  \item[\textbf{A1}.] For all $\oo \in \Oo$, $\Oo_{i,\oo} = \Pi_i(\oo)$.
  \item[\textbf{A2}.] For all $\oo \in \Oo$, if $\oo' \in \Pi_i(\oo)$,
 then ${\cal P}_i(\oo') = {\cal P}_i(\oo)$.
  \item[\textbf{A3}.] For all $j \in N, \oo, \oo' \in \Oo$,
      $\Pi_i(\oo) \cap \Pi_j(\oo') \in \F_{i,\oo}$.
\end{description}
Furthermore, we make the following joint assumption on players'
interpretations and information partitions:
\begin{description}
  \item[\textbf{A4}.] For all $\oo \in \Oo$, $i \in N$, and
            primitive proposition $p \in \Phi$, $\Pi_i(\oo)
      \cap [[p]]_i \in \F_{i,\oo}$.
\end{description}
These are all standard assumptions.  A1 says that the set of states to
which player $i$ assigns probability at state $\oo$ is just the set
$\Pi_i(\oo)$ of worlds that $i$ considers possible at state $\oo$.
A2 says that the probability space used is the same at all the worlds in
a cell of player $i$'s partition.  Intuitively, this says that
player $i$ knows his probability space. Informally, A3 says that player
$i$ can assign a probability to each of $j$'s cells, given his information.
A4 says that primitive propositions (as interpreted by player $i$) are
measurable
according to player $i$.

\subsection{Prior-generated beliefs}

One assumption that we do not necessarily make, but want to examine in
this framework, is the common-prior assumption.
The common-prior assumption is an instance of a more general assumption,
that beliefs are generated from a prior, which we now define.  The
intuition is that players start with a prior probability; they then
update the prior in light of their information.  Player $i$'s information
is captured by her partition  $\Pi_i$. Thus, if $i$'s prior is $\nu_i$,
then we would expect $\mu_{i,\oo}$ to be $\nu_i(\cdot \mid \Pi_i(\oo))$.

\begin{definition}
An epistemic probability structure $M = (\Oo, (\Pi_j)_{j \in
N}, ({\cal P}_j)_{j \in N}, (\pi_j)_{j \in N})$ has
\emph{prior-generated beliefs}
(\emph{generated by $(\F_1,\nu_1),\ldots, (\F_n,\nu_n)$})
if, for each player $i$, there exist
probability spaces $(\Oo,\F_i,\nu_i)$ such that
\begin{itemize}
\item for all $i,j \in N$ and $\oo \in \Oo$, $\Pi_j(\oo) \in \F_i$;
\item for all $i \in N$ and $\oo \in \Oo$, ${\cal P}_i(\oo) =
    (\Pi_i(\oo), \F_i \mid \Pi_i(\oo), \mu_{i,\oo})$,
    where $\F_i\mid\Pi_i(\oo)$ is the restriction of $\F_i$ to $\Pi_i(\oo)$,%
\footnote{Recall that the restriction of ${\cal F}_i$ to
$\Pi_i(\oo)$ is the $\sigma$-algebra $\{B \cap \Pi_i(\oo): B \in {\cal
F}_i\}$.} %
and $\mu_{i,\oo}(E) =
    \nu_i(E \mid \Pi_i(\oo))$ for all $E \in
    \F_i \mid \Pi_i(\oo)$ if $\nu_i(\Pi_i(\oo)) > 0$.
(There are no constraints on $\nu_{i,\oo}$ if $\nu_i(\Pi_i(\oo)) = 0$.)
\end{itemize}
\end{definition}
It is easy to check that if $M$ has prior-generated beliefs, then
$M$ satisfies A1, A2, and A3.
More interestingly for our purposes, the converse also holds for a large
class of structures.
Say that a structure is
\emph{countably partitioned} if for each player
$i$, the information partition $\Pi_i$ has countably many elements,
i.e., $\Pi_i$ is a finite or countably infinite collection of subsets of
$\Omega$.

\begin{proposition}\label{prop:priorgenerated}
If a structure $M$ has
prior-generated beliefs, then
$M$ satisfies A1, A2, and A3. Moreover, every countably partitioned structure that
satisfies A1, A2, and A3 is one with prior-generated beliefs,
with the priors $\nu_i$ satisfying $\nu_i(\Pi_i(\oo))>0$ for each player
$i \in N$ and state $\oo \in \Oo$.%
\end{proposition}
\begin{proof}
The first part is immediate. To prove the second claim, suppose that $M$
is a structure satisfying A1--A3. Let $\F_i$ be the
unique algebra generated by $\cup_{\oo \in \Oo} \F_{i,\oo}$. To define
$\nu_i$, if there are $N_i < \infty$ cells in the partition $\Pi_i$,
define $\nu_i(\oo) = \tfrac{1}{N_i} \mu_{i,\oo}(\oo)$. Otherwise, if the
collection $\Pi_i$ is countably infinite, order the elements of $\Pi_i$
as $p_i^1, p_i^2, \ldots$.  Choose some state $\oo_k \in p_i^k$ for each
$k$, with associated probability space ${\cal P}_i(\oo_k) =
(\Oo_{i,\oo_k}, \F_{i,\oo_k}, \mu_{i,\oo_k})$.
By A2, each choice of $\oo_k$ in $p_i^k$ gives the
same probability measure $\mu_{i,\oo_k}$.  Define $\nu_i = \sum_k
\tfrac{1}{2^k} \mu_{i,\oo_k}$.
It is easy to see that
$\nu_i$ is a probability measure on $\Oo$, and that
$M$ is generated by $(\F_1,\nu_1),\ldots, (\F_n,\nu_n)$.
\end{proof}

Note that the requirement that that $M$ is countably partitioned is necessary to
ensure that we can have
$\nu_i(\Pi_i(\oo))>0$ for each player $i$ and state $\oo$.

In light of Proposition~\ref{prop:priorgenerated}, when it is
convenient, we will talk of a structure satisfying A1--A3 as being
\emph{generated by} $(\F_1,\nu_1), \ldots, (\F_n,\nu_n)$.

The \emph{common-prior assumption} discussed in the introduction is essentially just the special case of
prior-generated beliefs where all the
priors are identical.
\commentout{
We make one additional technical assumption. To state this
assumption, we need one more definition. A state $\oo' \in
\Oo$ is \emph{$G$-reachable} from $\oo \in \Oo$, for $G \subseteq N$, if
there exists a
sequence $\oo_0, \ldots, \oo_m$ in $\Oo$ with $\oo_0 = \oo$ and
$\oo_m = \oo'$, and $i_1, \ldots, i_m \in G$ such that
$\oo_\ell \in \Pi_{i_\ell}(\oo_{\ell -1})$. Denote by $R_G(\oo)
\subseteq \Oo$ the set of states $G$-reachable from $\oo$.

\begin{definition} \label{def:CPA}
An epistemic probability structure $M = (\Oo, (\Pi_j)_{j \in
N}, ({\cal P}_j)_{j \in N}, (\pi_j)_{j \in N})$ satisfies the
\emph{common-prior assumption} (\emph{CPA}) if there exists a
probability space $(\Oo,\F,\nu)$ such that $M$ has prior-generated
beliefs generated by
$((\F,\nu), \ldots, (\F,\nu))$, and $\nu(R_N(\oo)) > 0$ for all $\oo \in
\Oo$.
\end{definition}
As shown by Halpern~\citeyear{Halpern_2002}, the assumption that $\nu(R_N(\oo))
> 0$ for each $\oo \in \Oo$ is needed for Aumann's
\citeyear{Au} impossibility result.
}

\subsection{Capturing ambiguity}

We use epistemic probability structures to give meaning to formulas.
Since primitive propositions are interpreted relative to players,
we must allow the interpretation of arbitrary formulas to depend on the
player as well. Exactly how we do this depends on what further assumptions we
make about what
players know about each other's interpretations.  There are many
assumptions that could be made. We focus on two of them here, ones that
we believe arise in applications of interest,
and then
reconsider them under the assumption that there may be some ambiguity
about the partitions.

\paragraph{Believing there is no ambiguity}

The first approach is appropriate for situations where players
may interpret statements differently, but it does not occur to them
that there is another way of interpreting the statement.
Thus, in this
model, if there is a public announcement, all players will think that
their interpretation of the announcement is common knowledge.
We write $(M,\oo,i) \vDasho \ffi$ to denote that $\ffi$ is true at
state $\oo$ according to player $i$
(that is, according to $i$'s interpretation of the primitive
propositions in $\ffi$).
The superscript $\ou$ denotes
\emph{outermost scope}, since the formulas are interpreted relative to
the ``outermost'' player, namely the player $i$ on the left-hand side of
$\vDasho$.  We define $\vDasho$, as usual, by induction.

If $p$ is a primitive proposition,
\[
(M,\oo,i) \vDasho p  \  \text{iff} \   \pi_i(\oo)(p) = \mathbf{true}.
\]
This just says that player $i$ interprets a primitive proposition $p$
according to his interpretation function $\pi_i$.  This clause is common
to all our approaches for dealing with ambiguity.

For conjunction and negation, as is standard,
\begin{gather*}
(M,\oo,i) \vDasho \neg \ffi  \  \text{iff}  \
(M,\oo,i) \not{\vDash}^{\ou} \ffi,\\
(M,\oo,i) \vDasho  \ffi \wedge \psi \   \text{iff} \
  (M,\oo,i) \vDasho \ffi \text{ and } (M,\oo,i) \vDasho \psi.
\end{gather*}

Now consider a probability formula of the form $a_1 \pr_j(\ffi_1) +
\ldots + a_k \pr_j(\ffi_k) \geq b$.
The key feature that distinguishes this semantics is how $i$ interprets
$j$'s beliefs.
This is where we capture the intuition that it does not
occur to $i$ that there is another way of interpreting the formulas other
than the way she does.  Let
\[[[\ffi]]^{\ou}_i = \{\oo: (M,\oo,i) \vDasho \ffi\}.\]
Thus, $[[\ffi]]^{\ou}_i$ is the event consisting of the set of states
where $\ffi$ is true, according to $i$.
Note that A1 and A3 guarantee that the restriction of $\Oo_{j,\oo}$ to $\Pi_i(\oo)$ belongs to $\F_{i,\oo}$.
Assume inductively that $[[\ffi_1]]^{\ou}_i \cap
\Omega_{j,\oo}, \ldots, [[\ffi_k]]^{\ou}_i \cap  \Omega_{j,\oo}
\in \F_{j,\oo}$.  The base case of this induction, where
$\ffi$ is a primitive proposition, is immediate from A3 and A4,
and the
induction assumption clearly extends to negations and conjunctions. We now define
\begin{multline*}\label{eq:prob_formula}
(M,\oo,i) \vDasho a_1 \pr_j(\ffi_1) + \ldots + a_k \pr_j(\ffi_k)
\geq b   \ \text{iff} \ \nonumber\\ a_1
\mu_{j,\oo}([[\ffi_1]]^{\ou}_i \cap \Omega_{j,\oo}) + \ldots + a_k
\mu_{j,\oo}([[\ffi_k]]^{\ou}_i \cap \Omega_{j,\oo}) \geq b.
\end{multline*}
Note that it easily follows from A2 that
$(M,\oo,i) \vDasho a_1 \pr_j(\ffi_1) + \ldots + a_k \pr_j(\ffi_k)
\geq b$ if and only if
$(M,\oo',i) \vDasho a_1 \pr_j(\ffi_1) + \ldots + a_k \pr_j(\ffi_k)
\geq b$ for all $\oo' \in \Pi_j(\oo)$.  Thus,
$[[a_1 \pr_j(\ffi_1) + \ldots + a_k \pr_j(\ffi_k) \geq b]]_i$ is a union of
cells of $\Pi_j$, and hence
$[[a_1 \pr_j(\ffi_1) + \ldots + a_k \pr_j(\ffi_k) \geq b]]_i \cap
\Omega_{j,\oo} \in \F_{j,\oo}$.

With this semantics,
according to player $i$, player $j$ assigns $\ffi$ probability
$b$ if and
only if the set of worlds where $\ffi$ holds according to $i$ has
probability $b$ according to $j$.
Intuitively, although $i$ ``understands'' $j$'s probability space, player
$i$ is not aware that $j$ may interpret $\ffi$ differently from the way
she ($i$) does. That $i$ understands $j$'s probability space
is plausible if we assume that there is a common prior and that $i$
knows $j$'s partition (this knowledge is embodied in the assumption that
$i$ intersects $[[\ffi_k]]^{\ou}_i$ with $\Omega_{j,\oo}$ when assessing
what probability $j$ assigns to $\ffi_k$).\footnote{Note that at state $\oo$, player $i$ will not in general know that it is state
$\oo$.  In particular, even if we assume that $i$ knows which element of
$j$'s partition contains $\oo$, $i$ will not in general know which of
$j$'s cells describes $j$'s current information.  But we assume that
$i$ does know that \emph{if} the state is $\oo$, then $j$'s information is
described by $\Oo_{j,\oo}$.  Thus, as usual, ``$(M,i,\oo) \vDasho
\ffi$'' should
perhaps be understood as ``according to $i$, $\ffi$ is true if the
actual world is $\oo$''. This interpretational issue arises even
without
ambiguity in the picture.}

Given our interpretation of probability formulas, the interpretation of
$B_j \ffi$ and $\EB^k \ffi$ follows.
For example,
\[
(M,\oo,i) \vDasho B_j \ffi \ \text{iff} \ \mu_{j,\oo}([[\ffi]]^{\ou}_i) = 1.
\]
For readers more used to belief defined in terms of a possibility
relation, note that if the probability measure $\mu_{j,\oo}$ is
\emph{discrete} (i.e., all sets are $\mu_{j,\oo}$-measurable, and
$\mu_{j,\oo}(E) = \sum_{\oo' \in E} \mu_{j,\oo}(\oo')$ for all
subsets $E \subset \Pi_j(\oo)$),
we can define
$\B_j = \{(\oo,\oo') : \mu_{j,\oo}(\oo') > 0\}$; that is, $(\oo,\oo')
\in \B_j$ if, in state $\oo$, agent $j$ gives state $\oo'$ positive
probability.  In that case,
$(M,\oo,i) \vDasho B_j \ffi$ iff $(M,\oo',i) \vDasho \ffi$
for all $\oo'$ such that $(\oo,\oo') \in \B_j$.
That is, $(M,\oo,i) \vDasho B_j \ffi$ iff $\ffi$ is true according to
$i$ in all the worlds to which $j$ assigns positive probability at $\oo$.

It is important to note that $(M,\oo,i) \vDash \ffi$ does not imply
$(M,\oo,i) \vDash B_i \ffi$: while $(M,\oo,i) \vDasho \ffi$ means
``$\ffi$ is true at $\oo$ according to $i$'s interpretation,'' this does
not mean that $i$ believes $\ffi$ at state $\oo$. The reason is that $i$
can be uncertain as to which state is the actual state. For $i$ to
believe $\ffi$ at $\oo$, $\ffi$ would have to be true (according to
$i$'s interpretation) at all states
to which $i$ assigns positive probability.

Finally, we define
\[
(M,\oo,i) \vDasho \CB_G \ffi  \  \text{iff}  \
(M,\oo,i) \vDasho \EB^k_G \ffi  \  \text{for $k = 1,2, \ldots$}
\]
for any nonempty subset $G \subseteq N$ of players.

\paragraph{Awareness of possible ambiguity}

We now consider the second way of interpreting formulas.  This is
appropriate for players who realize that other players may interpret
formulas differently.  We write
$(M,\oo,i) \vDashi \ffi$ to denote that $\ffi$ is true at state $\oo$
according to player $i$ using this interpretation, which is called
\emph{innermost scope}. The definition of $\vDashi$ is identical to
that of $\vDasho$ except for the interpretation of probability formulas.
In this case, we have
\begin{multline*}
(M,\oo,i) \vDashi a_1 \pr_j(\ffi_1) + \ldots + a_k \pr_j(\ffi_k)
\geq b  \  \text{iff} \ \nonumber \\ a_1
\mu_{j,\oo}([[\ffi_1]]^\inner_j \cap \Omega_{j,\oo}) + \ldots + a_k
\mu_{j,\oo}([[\ffi_k]]^\inner_j \cap \Omega_{j,\oo}) \geq b,
\end{multline*}
where $[[\ffi]]^\inner_j$ is the set of states $\oo$
such that $(M,\oo,j) \vDashi \ffi$.
Hence, according to player $i$, player $j$ assigns $\ffi$ probability
$b$ if and
only if the set of worlds where $\ffi$ holds according to $j$ has
probability $b$ according to $j$.  Intuitively, now $i$ realizes that
$j$ may interpret $\ffi$ differently from the way that she ($i$) does,
and thus assumes that $j$ uses his ($j$'s) interpretation to evaluate the
probability of $\ffi$.
Again, in the case that $\mu_{j,\oo}$ is discrete, this means that
$(M,\oo,i) \vDashi B_j\ffi$ iff $(M,\oo',j) \vDashi \ffi$ for all $\oo'$
such that $(\oo,\oo') \in \B_j$.

Note for future reference that if $\ffi$ is a probability formula or a
formula of the form
$\CB_G \ffi'$, then it is easy to see that
$(M,\oo,i) \vDashi \ffi$ if and only if $(M,\oo,j)\vDashi \ffi$; we
sometimes write $(M,\oo) \vDashi \ffi$ in this case.
Clearly, $\vDasho$ and $\vDashi$ agree in the common-interpretation
case, and we can write $\vDash$.
There is a sense in which innermost scope is able to capture the intuitions
behind outermost scope.  Specifically, we can capture the intuition that
player $i$ is convinced that all players interpret everything just as he
($i$) does by assuming that in all worlds $\oo'$ that player $i$
considers possible, $\pi_i(\oo') = \pi_j(\oo')$ for all players $j$.

\paragraph{Ambiguity about information partitions}

Up to now, we have assumed that players ``understand'' each other's
probability spaces.  This may not be so reasonable in the presence of
ambiguity and prior-generated beliefs.  We want to model the following
type of situation. Players receive information, or signals, about the
true state
of the world, in the form of strings (formulas). Each player understands
what signals he and other players receive in different states of the
world, but players may interpret signals differently. For instance, player
$i$ may understand that $j$ sees a red car if $\oo$ is the true state of
the world, but $i$ may or may not be aware that $j$ has a different
interpretation of ``red'' than $i$ does. In the latter case, $i$ does not have a full
understanding of $j$'s information structure.

We would like to think of a player's information as being characterized by a formula (intuitively,
the formula that describes the signals received).  Even if the formulas
that describe each information set are commonly known, in the presence
of ambiguity, they might be interpreted differently.

To make this precise, let $\Phi^*$ be the set of formulas that is
obtained from $\Phi$ by closing off under negation and conjunction.
That is, $\Phi^*$ consists of all propositional formulas that can be
formed from the primitive propositions in $\Phi$.
Since the formulas in $\Phi^*$ are not composed of probability
formulas, and thus do not involve any reasoning about interpretations,
we can extend the function $\pi_i(\cdot)$ to $\Phi^*$ in a
straightforward way, and write $[[\ffi]]_i$ for the set of the states of
the world where the formula $\ffi \in \Phi^*$ is true according to $i$.

The key new assumption that we make to model players' imperfect understanding
of the other players' probability spaces is that $i$'s partition cell at
$\oo$
is described by a formula $\ffi_{i,\oo} \in \Phi^*$.
Roughly speaking, this means that $\Pi_i(\omega)$ should consist of all
states where the formula $\ffi_{i,\oo}$ is true.
More precisely, we take $\Pi_i(\omega)$ to consist of all states
where $\phi_{i,\oo}$ is true according to $i$.
If player $j$ understands that $i$ may be using a different
interpretation than he does (i.e., the appropriate semantics are the
innermost-scope semantics), then $j$ correctly infers that the set of
states that $i$ thinks are possible in $\oo$ is $\Pi_i(\oo) =
[[\ffi_{i,\oo}]]_i$. But if $j$ does not understand that $i$ may
interpret formulas in a different way (i.e., under outermost scope),
then he thinks that the set of states that $i$ thinks are possible in
$\oo$ is given by $[[\ffi_{i,\oo}]]_j$.  Of course,
$[[\ffi_{i,\oo}]]_j$ does not in general coincide with $\Pi_i(\oo)$.
Indeed, $[[\ffi_{i,\oo}]]_j$ may even be empty.  If this happens, $j$
might well wonder if $i$ is interpreting things the same way that he
($j$) is.
In any case, we
require that $j$ understand that these formulas form a partition and
that $\omega$ belongs to $[[\ffi_{i,\oo}]]_j$.
Thus, we consider structures that
satisfy A1--A5, and possibly A6 (when we use outermost scope semantics).
\begin{description}
  \item[\textbf{A5.}] For each $i \in N$ and $\oo \in \Omega$, there is
  a formula $\ffi_{i,\oo} \in \Phi^*$ such that $\Pi_i(\oo) =
  [[\ffi_{i,\oo}]]_i$.
  \item[\textbf{A6.}] For each $i, j \in N$, the collection
  $\{[[\ffi_{i,\oo}]]_j: \oo \in \Oo\}$ is a partition of $\Omega$
and for all $\oo \in \Oo$, $\oo \in [[\ffi_{i,\oo}]]_j$.
\end{description}
Assumption A6 ensure that the signals for player $i$ define an
information partition according to every player $j$ when we consider the
outermost scope semantics. With innermost scope, this already follows
from A5 and the definition of $\Pi_i(\oo)$.

We can now define analogues of outermost scope and innermost scope
in the presence of ambiguous information. Thus, we define two more truth
relations, $\vDashos$ and $\vDashis$.
(The ``ai'' here stands for ``ambiguity of information''.)
The only difference between $\vDashos$ and $\vDasho$ is in the semantics of probability formulas. In giving the
semantics in a structure $M$, we assume that $M$ has prior-generated
beliefs, generated by $(\F_1,\nu_1),\ldots, (\F_n,\nu_n)$.
As we observed in Proposition \ref{prop:priorgenerated}, this
assumption is without loss of generality
as long as the structure is countably partitioned. However, the choice of prior beliefs \emph{is} relevant, as we shall
see, so we have to be explicit about them.
When $i$ evaluates $j$'s probability at a state
$\oo$, instead of using $\mu_{j,\oo}$, player $i$ uses $\nu_j(\cdot \mid
[[\ffi_{j,\oo}]]_i)$.  When $i=j$, these two approaches agree,
but in general they do not. Thus, assuming that $M$ satisfies A5 and A6 (which are the appropriate assumptions for the outermost-scope semantics),
we have
\[
\begin{array}{ll}
(M,\oo,i) \vDashos a_1 \pr_j(\ffi_1) + \ldots + a_k \pr_j(\ffi_k)
\geq b \   \text{iff} \\ a_1
\nu_j([[\ffi_1]]^\ous_i \mid [[\ffi_{j,\oo}]]^\ous_i) + \ldots\\
\ \ \  + a_k \nu_j([[\ffi_k]]^\ous_i \mid [[\ffi_{j,\oo}]]^\ous_i) \geq
b,
\end{array}
\]
where $[[\psi]]^\ous_i = \{\oo': (M,\oo,i) \vDashos \psi\}$.

That is, at $\oo \in \Oo$, player $j$ receives the
information (a string) $\ffi_{j,\oo}$, which he interprets as
$[[\ffi_{j,\oo}]]_j$. Player $i$ understands that $j$ receives the
information $\ffi_{j,\oo}$ in state $\oo$, but interprets this as
$[[\ffi_{j,\oo}]]_i$. This models a situation such as the following. In
state $\oo$, player $j$ sees a red car, and thinks possible all states of
the world where he sees a car that is red (according to $j$).
Player $i$ knows that at world $\oo$ player $j$ will see a red car
(although she may not know that the actual world is $\oo$, and thus does
not know what color of car player $j$ actually sees).  However, $i$ has a
somewhat different interpretation of ``red car''
(or, more precisely, of $j$ seeing a red car)
than $j$; $i$'s
interpretation corresponds to the event $[[\ffi_{j,\oo}]]_i$.  Since $i$
understands that $j$'s beliefs are determined by conditioning her prior
$\nu_j$ on her information, $i$ can compute what she believes $j$'s
beliefs are.

We can define $\vDashis$ in an analogous way. Thus, the semantics
for formulas that do not involve probability formulas are as given by
$\vDashi$, while the semantics of probability formulas is defined as
follows (where $M$ is assumed to satisfy A5, which is the appropriate
assumption for the innermost-scope semantics):
\[
\begin{array}{ll}
(M,\oo,i) \vDashis a_1 \pr_j(\ffi_1) + \ldots + a_k \pr_j(\ffi_k)
\geq b  \  \text{iff} \\
a_1
\nu_j([[\ffi_1]]^\inners_j \mid [[\ffi_{j,\oo}]]^\inners_j) + \ldots \\
\ \ \ + a_k \nu_j([[\ffi_k]]^\inners_j \mid [[\ffi_{j,\oo}]]^\inners_j) \geq b.
\end{array}
\]
Note that although we have written $[[\ffi_{j,\oo}]]^\inners_i$, since
$\ffi_{j,\oo}$ is a propositional formula, $[[\ffi_{j,\oo}]]^\inners_i =
[[\ffi_{j,\oo}]]^\ous_i = [[\ffi_{j,\oo}]]^\ou_i =
[[\ffi_{j,\oo}]]^\inner_i$.
It is important that $\ffi_{j,\oo}$ is a propositional formula here;
otherwise, we would have circularities in the definition, and would
somehow need to define $[[\ffi_{j,\oo}]]^\inners_i$.

Again, here it may be instructive to consider the definition of
$B_j \ffi$ in the case that $\mu_{j,\oo}$ is discrete for all $\oo$.
In this case, $\B_j$ becomes the set $\{(\oo,\oo'): \nu_j( \oo' \mid
[[\ffi_{j,\oo}]]^\inners_j) > 0$.
That is, state $\oo'$ is considered possible by player $j$ in state
$\oo$ if agent $j$ gives $\oo'$ positive probability after
conditioning his prior $\nu_j$ on (his
interpretation of) the information $\ffi_{j,\oo}$ he receives in state $\oo$.
With this definition of $\B_j$, we have, as expected,
$(M,\oo,i) \vDashis B_j \ffi$ iff $(M,\oo',i) \vDashis \ffi$ for all
$\oo'$ such that $(\oo,\oo') \in \B_j$.

The differences in the different semantics
arise only when we consider
probability formulas. If we go back to our example with the red car, we now have a situation
where player $j$ sees a red car in state $\oo$, and thinks possible all
states where he sees a red car. Player $i$ knows that in state $\oo$,
player $j$ sees a car that he ($j$) interprets to be red, and that this
determines his posterior. Since $i$ understands $j$'s notion of seeing a
red car, she has a correct perception of $j$'s posterior in each
state of the world. Thus, the semantics for $\vDashis$ are identical to
those for $\vDashi$
(restricted to the class of structures with prior-generated beliefs that
satisfy A5), though the information partitions are not
predefined, but rather generated by the signals.

Note that, given an epistemic structure $M$ satisfying A1--A4, there are many
choices for $\nu_i$ that allow $M$ to be viewed as being generated by
prior beliefs.  All that is required of $\nu_j$ is that
for all $\oo \in \Oo$ and $E \in \F_{j,\oo}$ such that $E \subseteq [[\ffi_{j,\oo}]]^\ous_j$, it holds that
$\nu_j(E \cap [[\ffi_{j,\oo}]]^\ous_j )/\nu_j([[\ffi_{j,\oo}]]^\ous_j) = \mu_{j,\oo}(E)$.
\commentout{
It easily follows that if
$E, E' \subseteq [[\ffi_{j,\oo}]]^\ous_j = \Pi_j(\oo)$ and $E, E' \in\F_{j,\oo}$,
then if $\nu_j$ and $\nu'_j$ generate the same posterior beliefs for
$j$ and $\nu_j(E') \ne 0$, then $\nu_j'(E') \ne 0$, and
$\nu_j(E)/\nu_j(E') = \nu_j'(E)/\nu_j'(E')$.
}
However, because $[[\ffi_{j,\oo}]]^\ous_i$ may not be a subset of
$[[\ffi_{j,\oo}]]^\ous_j = \Pi_j(\oo)$,
we can have two prior probabilities $\nu_j$ and $\nu_j'$ that
generate the same posterior beliefs for $j$, and still have
$\nu_j([[\ffi_k]]^\ous_i \mid [[\ffi_{j,\oo}]]^\ous_i) \ne \nu_j'([[\ffi_k]]^\ous_i \mid [[\ffi_{j,\oo}]]^\ous_i)$ for some formulas $\ffi_k$.  Thus, we must
be explicit about our choice of priors here.

\section{Common interpretations suffice}\label{sec:equivalence}

In this section, we show in there is a sense in which we do not need
structures with ambiguity.
Specifically, we show that the same
formulas are valid in common-interpretation structures as in structures
that do not have a common interpretation, no matter what semantics we
use, even if we have ambiguity about information partitions.

To make this precise, we need some notation.
Fix a nonempty, countable set $\Psi$ of primitive propositions, and let
$\M(\Psi)$ be the class of all structures that satisfy A1--A4 and that are
defined over some nonempty subset $\Phi$ of $\Psi$ such that $\Psi
\setminus \Phi$ is countably infinite.\footnote{Most of our results hold if we just consider the set of
structures defined over some fixed set $\Phi$ of primitive propositions.
However, for one of our results, we need to be able to add fresh
primitive propositions to the language.  Thus, we allow the
set $\Phi$ of primitive propositions to vary over the
structures we consider, but require $\Psi \setminus \Phi$ to be countably
infinite so that there are always ``fresh'' primitive propositions that
we can add to the language.}
Given a subset $\Phi$ of $\Psi$, a formula $\ffi \in {\cal L}_n^C(\Phi)$, and a structure $M \in \M(\Psi)$ over $\Phi$, we say that
$\ffi$ is \emph{valid in $M$ according to outermost scope},
and write $M \vDasho \ffi$, if $(M,\oo,i)\vDasho \ffi$ for all
$\oo \in \Oo$ and $i \in N$.
Given $\ffi \in \Psi$, say that  $\ffi$ is \emph{valid according to
outermost scope} in a class $\N \subseteq \M(\Psi)$ of structures, and write $\N \vDasho \ffi$,
if $M \vDasho \ffi$ for all $M \in \N$ defined over a set $\Phi \subset \Psi$ of primitive
propositions that includes all the primitive
propositions that appear in $\ffi$.

We get analogous definitions by replacing $\vDasho$ by
$\vDashi$, $\vDashos$ and $\vDashis$ throughout (in the latter two
cases, we have to restrict $\N$ to structures that satisfy
A5 and A6 or just A5, respectively, in addition to A1--A4).
Finally, given a class of structures
$\N$, let $\N_c$ be the subclass of $\N$ in
which players have a common interpretation. Thus, $\M_c(\Psi)$ denotes the structures in $\M(\Psi)$ with a common interpretation.
Let $\Mis(\Psi)$ denote all structures in $\M(\Psi)$ with prior-generated
beliefs that satisfy A5 and A6 (where we assume that the prior $\nu$
that describes the initial beliefs is given explicitly).\footnote{For ease of exposition, we assume A6 even when dealing with
innermost scope.
}
\commentout{
Finally, let $\Mcpa(\Psi)$ (resp., $\Miscpa(\Psi)$)
consist of the structures in $\M(\Psi)$ (resp., $\Mis(\Psi)$) satisfying
the CPA.
}

\begin{theorem}\label{pro:equivalence}
For all formulas $\ffi \in {\cal L}_n^C(\Psi)$,
the following are equivalent:
\begin{itemize}
\item[(a)] $\M_c(\Psi) \vDash \ffi$;
\item[(b)] $\M(\Psi) \vDasho \ffi$;
\item[(c)] $\M(\Psi) \vDashi \ffi$;
\item[(d)] $\Mis_c(\Psi) \vDash \ffi$;
\item[(e)] $\Mis(\Psi) \vDashos \ffi$;
\item[(f)] $\Mis(\Psi) \vDashis \ffi$.
\end{itemize}
\end{theorem}

\begin{proof}
Since the set of structures with a common interpretation
is a subset of the set of structures, it is immediate that (c) and (b)
both imply (a). Similarly, (e) and (f) both imply (d).
The fact that (a) implies (b) is also immediate. For
suppose that $\M_c(\Psi) \vDash \ffi$ and that
$M = (\Oo, (\Pi_j)_{j \in N},
({\cal P}_j)_{j \in N}, (\pi_j)_{j \in N}) \in  \M(\Psi)$
is a structure
over a set $\Phi \subset \Psi$ of primitive propositions that contains
the primitive propositions that appear in $\ffi$. We must show
that $M \vDasho \ffi$.  Thus, we must show
that $(M,\oo,i) \vDasho \ffi$ for all $\oo \in \Oo$ and $i \in N$. Fix
$\oo \in \Oo$ and $i \in N$, and
let $M'_i = (\Oo, (\Pi_j)_{j \in N},
({\cal P}_j)_{j \in N}, (\pi'_j)_{j \in N})$, where
$\pi'_j = \pi_i$ for all $j$. Thus, $M'_i$ is a common-interpretation
structure over $\Phi$, where the
interpretation coincides with $i$'s interpretation in
$M$. Clearly $M'_i$ satisfies A1--A4, so $M'_i \in \M_c(\Psi)$. It is easy to
check that $(M,\oo,i) \vDasho \psi$ if and only if
$(M'_i,\oo,i) \vDash \psi$ for all states $\oo \in \Oo$ and all formulas
$\psi \in {\cal L}_n^C(\Phi)$. Since $M'_i \vDash \ffi$, we must have
that $(M,\oo,i) \vDasho \ffi$, as desired.

To see that (a) implies (c), given a structure
$M = (\Oo, (\Pi_j)_{j \in N}, ({\cal P}_j)_{j \in N}, (\pi_j)_{j \in N})
\in  \M(\Psi)$ over some set $\Phi \subset \Psi$ of primitive
propositions and a player $j \in N$, let $\Oo_j$ be a disjoint copy of
$\Oo$; that is, for every
state $\oo \in \Oo$, there is a corresponding state $\oo_j \in \Oo_j$.
Let $\Oo' = \Oo_1 \cup \ldots \cup \Oo_n$.
Given $E \subseteq \Oo$, let the corresponding subset $E_j \subseteq
\Oo_j$ be the set $\{\oo_j: \oo \in E\}$, and let $E'$ be the subset of
$\Oo'$ corresponding to $E$, that is, $E' = \{\oo_j: \oo \in E, j \in
N\}$.

Define $M' = (\Oo', (\Pi'_j)_{j \in N}, ({\cal P}'_j)_{j \in N},
(\pi'_j)_{j \in N})$, where  $\Oo' = \Oo_1 \cup \ldots \cup \Oo_n$ and, for
all $\oo \in \Oo$ and $i,j \in N$, we have
\begin{itemize}
\item $\Pi'_i(\oo_j) = (\Pi_i(\oo))'$;
\item $\pi_i(\oo_j)(p) = \pi_j(\oo)(p)$ for a primitive proposition $p
\in \Phi$;
\item ${\cal P}'_i(\oo_j) = (\Oo'_{i,\oo_j}, \F'_{i,\oo_j}, \mu'_{i,\oo_j})$, where $\Oo'_{i,\oo_j} = \Oo_{i,\oo}'$,
$\F'_{i,\oo_j} = \{E_\ell: E \in \F_{i,\oo}, \ell \in N\}$,
$\mu'_{i,\oo_j}(E_i) = \mu_{i,\oo}(E)$,
$\mu'_{i,\oo_j}(E_\ell) = 0$ if $\ell \ne i$.
\end{itemize}

Thus, $\pi_1 = \cdots = \pi_n$, so that $M'$ is a common-interpretation
structure; on a state $\oo_j$, these interpretations are
all determined by $\pi_j$.  Also note that the support of the probability
measure $\mu'_{i,\oo_j}$ is contained in $\Oo_i$, so for different
players $i$, the probability measures $\mu'_{i,\oo_j}$ have disjoint
supports. Now an easy induction on the structure of formulas shows that$(M',\oo_j) \vDash \psi$ if and only if $(M,\oo,j) \vDashi \psi$ for any formula $\psi \in {\cal L}_n^C(\Phi)$.  It easily
follows that if $M' \vDash \ffi$, then $M \vDashi \ffi$ for all $\ffi \in {\cal L}_n^C(\Phi)$.

The argument that (d) implies (e) is essentially identical to the
argument that (a) implies (b); similarly, the argument that
(d) implies (f) is essentially the same as the argument that (a) implies
(c).  Since $\Mis_c(\Psi) \subseteq \M_c(\Psi)$, (a) implies (d).  To show that (d)
implies (a), suppose that $\Mis_c(\Psi) \vDash \ffi$ for some formula
$\ffi \in {\cal L}_n^C(\Psi)$. Given a structure
$M = (\Oo, (\Pi_j)_{j \in N}, ({\cal P}_j)_{j \in N}, \pi)\in \M_c(\Psi)$ %
over a set $\Phi \subset \Psi$ of primitive propositions that includes
the primitive
propositions that appear in $\ffi$, we want to show that
$(M,\oo,i) \vDash \ffi$ for each state $\oo \in \Oo$ and player $i$.  Fix
$\oo$.  Recall that $R_N(\oo)$ consists of the set of states $N$-reachable
from $\oo$.  Let $M' = (R_N(\oo), (\Pi'_j)_{j \in N}, ({\cal P}'_j)_{j \in N}, \pi')$,
with $\Pi'_j$ and ${\cal P}'_j$ the restriction of $\Pi_j$ and ${\cal
P}_j$, respectively, to the states
in $R_N(\oo)$, be a structure over a set $\Phi'$ of primitive
propositions, where $\Phi'$ contains $\Phi$ and
new primitive propositions that we call
$p_{i,\oo}$ for each player $i$ and state $\oo \in
R_N(\oo)$.\footnote{This is the one argument that needs the assumption that the
set of primitive propositions can be different in different structures in
$\M(\Psi)$, and the fact that every $\Psi \setminus \Phi$ is countable.
We have
assumed for simplicity that the propositions $p_{i,\oo}$ are all in
$\Psi \setminus \Phi$, and that they can be chosen in such a way so that
$\Psi \setminus (\Phi \cup \{p_{i,\oo}: i \in \{1,\ldots,n\}, \oo \in
\Omega\})$  is
countable.}
Note that there are only countably many information sets in $R_N(\oo)$, so
$\Phi'$ is countable. Define $\pi'$ so that it agrees with $\pi$
(restricted to $R_N(\oo)$) on
the propositions in $\Phi$, and so that $[[p_{i,\oo}]]_i = \Pi_i(\oo)$.  Thus,
$M'$ satisfies A5 and A6. It is easy to check that, for all $\oo' \in R_N(\oo)$ and
all formulas $\psi \in {\cal L}_n^C(\Phi)$, we have that $(M,\oo',i)
\vDash \psi$ iff
$(M',\oo',i) \vDash \psi$.  Since $M' \vDash \ffi$, it follows that
$(M,\oo,i) \vDash \ffi$, as desired.
\end{proof}

From Theorem~\ref{pro:equivalence} it follows that for formulas in
${\cal L}_n^C(\Psi)$, we can get the same axiomatization with respect
to structures in $\M(\Psi)$ for both the $\vDasho$ and $\vDashi$
semantics; moreover, this axiomatization is the same as that for the
common-interpretation case.  An axiomatization for this case is already
given in \cite{FH3};
there is also a complete characterization of the complexity of
determining whether a formula is valid.

\commentout{
Things get more interesting if we consider
$\Mcpa(\Psi)$, the structures that satisfy the CPA.
Halpern  \citeyear{Halpern_2002}
provides an axiom that says that it cannot be common knowledge that
players disagree in expectation, and shows that it can be used to obtain
a sound and complete characterization of common-interpretation
structures with a common prior.  (The axiomatization is actually given
for common knowledge rather than common belief, but a similar result
holds with common belief.) By Proposition~\ref{pro:equivalence1},
the axiomatization remains sound for outermost-scope semantics if
we assume the CPA. However, using Example~\ref{exam:no_equiv}, we can
show that this is no longer the case for the innermost-scope semantics.
The set of formulas valid for innermost-scope semantics in the class of
structures satisfying the CPA is strictly between the set of formulas
valid in all structures and the set of formulas valid for
outermost-scope semantics in the class of structures satisfying the CPA. Finding an elegant complete
axiomatization remains an open problem.
}
However, the equivalence in Theorem \ref{pro:equivalence} does not
extend to subclasses of $\M$, $\M_c$, and $\Mis$. As shown in
our companion paper
\cite{HK13}, the equivalence result does not hold if we consider the
innermost scope semantics and
restrict attention to the subclasses of $\M$ and $\M_c$ that satisfy the
common-prior assumption.
We defer a further discussion of the modeling implications of this result to
Section~\ref{sec:concl}.

\section{A more general language}\label{sec:genlanguage}
Although, when considering innermost scope, we allowed for agents that
were sophisticated enough to realize that different agents might
interpret things in different ways, our syntax did not reflect that
sophistication.  Specifically, the language does not allow the modeler
(or the agents) to reason about how other agents interpret formulas.  Here
we consider a language that is rich enough to allow this.  Specifically,
we have primitive propositions of the form $(p,i)$, that can be
interpreted as ``$i$'s interpretation of $p$.''  With this extended
language, we do not need to have a different interpretation function
$\pi_i$ for each $i$; it suffices in a precise sense to use a single
(common) interpretation function.  We now make this precise, and show
that this approach is general enough to capture both outermost and
innermost scope.

More precisely, we consider the same syntax as in
Section~\ref{sec:syntax}, with the requirement that the set $\Phi$ of
primitive propositions have the form $\Phi' \times N$, for some set
$\Phi'$; that is, primitive propositions have the form $(p,i)$ for some
$p \in \Phi'$ and some agent $i \in N$.  We interpret these formulas
using a standard epistemic probability structure $M = (\Oo, (\Pi_j)_{j
\in N}, ({\cal P}_j)_{j \in N}, \pi)$, with a common interpretation
$\pi$, as in \cite{FH3}.  Thus, truth is no longer agent-dependent, so
we have only $(M,\oo)$ on the left-hand side of $\vDash$, not
$(M,\oo,i)$.
In
particular, if $(p,i)$ is a primitive proposition,
\[
(M,\oo) \vDash (p,i)  \  \text{iff} \   \pi(\oo)((p,i)) = \mathbf{true}.
\]
As expected, we have
\begin{multline*}\label{eq:prob_formula}
(M,\oo) \vDash a_1 \pr_j(\ffi_1) + \ldots + a_k \pr_j(\ffi_k)
\geq b   \ \text{iff} \ \nonumber\\ a_1
\mu_{j,\oo}([[\ffi_1]] \cap \Omega_{j,\oo}) + \ldots + a_k
\mu_{j,\oo}([[\ffi_k]] \cap \Omega_{j,\oo}) \geq b.
\end{multline*}
We no longer need to write $[[\ffi_j]]^{ou}_i$ or $[[\ffi_j]]^{in}_i$,
since all agents interpret all formulas the same way.

We now show how we can capture innermost and outermost scope using this
semantics.  Specifically, suppose that we start with an epistemic
probability structure $M = (\Oo, (\Pi_j)_{j \in N}, ({\cal P}_j)_{j \in
N}, (\pi_j)_{j \in N})$ over some set $\Phi$ of primitive propositions.
Consider the corresponding common-interpretation structure $M_c = (\Oo,
(\Pi_j)_{j \in N}, ({\cal P}_j)_{j \in N}, \pi)$ over $\Phi \times N$,
where $\pi(\oo)(p,i) = \pi_i(\oo)(p)$.    Thus, $M$ and $M_c$ are identical
except in the primitive propositions that they interpret, and how they
interpret them.  In $M_c$, the primitive proposition $(p,i) \in \Phi
\times N$ is interpreted the same way that $i$ interprets $p$ in $M$.

We can now define, for each formula $\phi$, two formulas $\phi_i^\inner$
and $\phi_i^\ou$ with the property that $(M,\oo,i) \vDashi \phi$ iff
$(M_c,\oo) \vDash \phi_i^\inner$ and $(M,\oo,i) \vDasho \phi$ iff
$(M_c,\oo) \vDash \phi_i^\ou$.  We start with $\phi_i^\inner$,
defining it by induction on structure:
\begin{itemize}
\item $p_i^\inner = (p,i)$
\item $(\psi \land \psi')_i^\inner = \psi_i^\inner \land (\psi'_i)^\inner$
\item $(a_1 \pr_j(\ffi_1) + \ldots + a_k \pr_j(\ffi_k) \geq b)_i^\inner
= a_1 \pr_j((\ffi_1)_j^\inner) + \ldots + a_k \pr_j((\ffi_k)_j^\inner) \geq b$
\item $(\CB_G \psi)_i^\inner = \CB_G (\land_{j \in G} B_j \psi_j^\inner)$.
\end{itemize}
Note that $\phi_i^\inner$ is independent of $i$ if $\phi$ is a probability
formula or of the form $\CB_G \psi$.  This is to be expected, since, as
we have seen, with innermost scope, the semantics of such formulas is
independent of $i$.
The definition of $(\CB_G \psi)_i^\inner$ is perhaps the only somewhat
surprising clause here; as we discuss after the proof of
Theorem~\ref{thm:trans} below, the more natural definition, $(\CB_G
\psi)_i^\inner = \CB_G(\psi_i^\inner)$, does not work.

For outermost scope, the first two clauses of the translation are
identical to those above; the latter two change as required for
outermost scope.  Thus, we get
\begin{itemize}
\item $p_i^\ou = (p,i)$
\item $(\psi \land \psi')_i^\ou = \psi_i^\ou \land (\psi'_i)^\ou$
\item $(a_1 \pr_j(\ffi_1) + \ldots + a_k \pr_j(\ffi_k) \geq b)_i^\ou
= a_1 \pr_j((\ffi_1)_i^\ou) + \ldots + a_k \pr_j((\ffi_k)_i^\ou) \geq b$
\item $(\CB_G \psi)_i^\ou = \CB_G (\psi_i^\ou)$.
\end{itemize}
Interestingly, here the natural definition of $(\CB_G \psi)_i^\ou$ does work.

\begin{theorem}\label{thm:trans} If $M$ is a probabilistic epistemic
structure over
$\Phi$ and $M_c$ is the corresponding common-interpretation structure over
$\Phi \times N$, then
\begin{itemize}
\item[(a)] $(M,\oo,i)\vDashi \phi$ iff $(M_c,\oo) \vDash \phi_i^\inner$;
\item[(b)] $(M,\oo,i)\vDasho \phi$ iff $(M_c,\oo) \vDash \phi_i^\ou$.
\end{itemize}
\end{theorem}

\begin{proof}
We prove the result by induction on the structure of $\phi$.  The
argument for outermost scope is completely straightforward, and left to
the reader.  The argument for innermost scope is also straightforward,
except for the case that $\phi$ has the form $\CB_G \psi$.  We now
consider this case carefully.

By definition,
$$\begin{array}{ll}
&(M_c,\oo) \vDash (\CB_G \psi)_i^\inner\\
\mbox{iff} &(M_c,\oo) \vDash \CB_G (\land_{j \in G} B_j \psi_j^\inner)\\
\mbox{iff} &(M_c,\oo) \vDash (\EB_G)^k  (\land_{j \in G} B_j \psi_j^\inner)
\mbox{ for $k= 1,2,3,\ldots$}.
\end{array}
$$
Note that, by definition, $(\EB_G\psi)_i^\inner = \land_{j \in G} B_j
\psi_j^\inner$.  Thus, by the induction hypothesis, it follows that
$$(M_c,\oo) \vDash \land_{j \in G} B_j
\psi_j^\inner \mbox{ iff }
(M,\oo,i) \vDashi \EB_G. $$
Now by a straightforward induction on $k$, we can show that
$$(M_c,\oo) \vDash \EB^k(\land_{j \in G} B_j \psi_j^\inner) \mbox{ iff }
(M,\oo,i) \vDashi \EB_G^{k+1}\psi.$$
That is,
\begin{equation}\label{eq:CB}
(M_c,\oo) \vDash \CB(\land_{j \in G} B_j \psi_j^\inner) \mbox{ iff }
(M,\oo,i) \vDashi \EB_G^k\psi \mbox{ for $k=2, 3, 4, \ldots$}.
\end{equation}

It immediately follows from \eqref{eq:CB}
that if $(M,\oo,i) \vDashi \CB_G \psi$, then
$(M_c,\oo) \vDash \CB(\land_{j \in G} B_j \psi_j^\inner)$.
The converse also follows
from (\ref{eq:CB}), once we show that
$(M,\oo,i) \vDashi \EB_G^2 \psi$ implies $(M,\oo,i)\vDashi \EB_G \psi$.
But this too follows easily since
$$\begin{array}{ll}
&(M,\oo,i) \vDashi \EB_G^2 \psi\\
\mbox{implies } &(M,\oo,i) \vDashi \land_{j \in G} B_j (\land_{j \in G} B_j \psi)\\
\mbox{implies } &(M,\oo,i) \vDashi \land_{j \in G} B_j ( B_{j} \psi)\\
 \mbox{iff } &(M,\oo,i) \vDashi \land_{j \in G} B_j \psi\\
\mbox{iff } &(M,\oo,i) \vDashi \EB \psi.
\end{array}$$
This completes the argument.
\end{proof}

To see why we need we need the more complicated definition
of $(\CB_G \psi)_i^\inner$, it is perhaps best to consider an example.
By definition,  $(\CB_{\{1,2\}} p)_1^\inner = \CB_{\{1,2\}} (B_1 (p,1)
\land  B_2(p,2))$.  By way of contrast, $\CB_{\{1,2\}} (p_1^\inner) =
\CB_{\{1,2\}} (p,1)$, which (using arguments similar in spirit to those
used above) can be shown to be equivalent to
$\CB_{\{1,2\}} (B_1 (p,1) \land B_2 (p,1))$.  They key point here is
whether we have $B_2 (p,1)$ or $B_2 (p,2)$.  We want the latter, which
is what we get from the more complicated translation that we use; it is
easy to show that the former does not give the desired result.  These
issues do not arise with outermost scope.

Theorem~\ref{thm:trans} shows that, from a modeler's point of view,
there is no loss in working with common-interpretations structures.
Any structure that uses ambiguous propositions can be converted to one
that uses unambiguous propositions of the form $(p,i)$.  In a sense,
this can be viewed as a strengthening of Theorem~\ref{pro:equivalence}.
Theorem~\ref{pro:equivalence} says that any formula that is satisfiable
using innermost or outermost semantics in the presence of ambiguity is
also satisfiable in a common-interpretation structure.  However, that
common-interpretation structure might be quite different from the
original structure.  Theorem~\ref{thm:trans} shows that if a formula
$\phi$ is true according to agent $i$ at a state $\oo$ in a structure $M$, then
a variant of $\phi$ (namely, $\phi_i^\inner$ or $\phi_i^\ou$) is true at
state $\oo$ in essentially the same structure.

Moreover, once we add propositions of the form $(p,i)$ to the language,
we have a great deal of additional expressive power.  For example, we
can say directly that agent $i$ believes that all agents interpret $p$
the same way that he does by writing $B_i (\land_j ((p,i) \iff (p,j)))$.
We can also make more complicated statements, such as ``agent $i$
believes that agents $j$ and $k$ interpret $p$ the same way, although
they interpret $p$ differently from him: $B_i((p,j) \iff (p,k)) \land
\neg B_i((p,i)\iff (p,j))$.  Clearly, far more subtle relationships
among agents' interpretations of primitive propositions can be expressed
in this language.

\commentout{
So why do we bother using models with ambiguity?  Perhaps the main
reason is that it allows us to describe the situation from the agent's
point of view.  For example, if we are dealing with outermost scope, an
agent does not realize that there are other interpretations possible
other than the one he is using.   Thus, the simpler language more
}

\section{Discussion} \label{sec:concl}

\commentout{
We have defined a logic for reasoning about ambiguity, and considered
the tradeoff between having a common prior (so that everyone starts out
with the same belief) and having a common interpretation (so that
everyone interprets all formulas the same way).  We showed that, in a
precise sense, allowing different beliefs is more general than allowing
multiple interpretations.  But we view that as a feature, not a
weakness, of considering ambiguity.  Ambiguity can be viewed as a reason
for differences of beliefs; as such, it provides some structure to these
differences.
}

We have defined a logic for reasoning about ambiguity, and then showed
that, in two senses, we really do not need structures with ambiguity:
(1) the same axioms hold whether or not we have ambiguity, and (2) we
can use a richer language to talk about the ambiguity, while giving an
unambigious interpretation to all formulas.
So why do we bother using structures with ambiguity?  Perhaps the main
reason is that it allows us to describe the situation from the agent's
point of view.  For example, if we are dealing with outermost scope, an
agent does not realize that there are other interpretations possible
other than the one he is using.   Thus, the simpler language more
directly captures agents' assertions.
Similarly, a structure with
ambiguity may more accurately describe a situation than a structure
with a common interpretation.
We thus believe that structures with
ambiguity will prove to be a useful addition to a modeler's toolkit.
In any case, whatever modeling framework and language is used, it is
clear that we need to take ambiguity into account, and reason explicitly
about it.

There are two extensions of our framework that we have not considered.
First, we model ambiguity by allowing a formula to be interpreted differently
by different agents, we assume that each individual agent disambiguates
each formula.  That is, no agent says ``I'm not sure how to disambiguate
$\phi$.  It could correspond to the $U$ of worlds, or it could
correspond to $U'$; I'm not sure which is right.''  As we mentioned
earlier, this view is closer to that of Lewis \citeyear{Lewis82} and
Kuijer \citeyear{Kui13}.  It would involve a nontrivial change to our
framework to capture this.
Second, we have allowed only ambiguity about the meaning of primitive
propositions (which then extends to ambiguity about the meaning of
arbitrary formulas).  But we have not considered ambiguity about the
meaning of belief; for example, $i$ might interpret belief in $\phi$ terms of
having a proof of $\phi$ in some axiom system, while $j$ might use a
possible-worlds interpretation (as we do in this paper).  Capturing this
seems interesting, but quite difficult.  Indeed, even without ambiguity,
it is not nontrivial to design a logic that captures various
resource-bounded notions of belief.  (See \cite{FHMV}[Chapters 9--10]
for more on this topic.)

\paragraph{Acknowledgments:} %
We thank Moshe Vardi and the anonymous reviewers of this paper for helpful
comments.
Halpern's work was supported
in part by NSF grants IIS-0534064, IIS-0812045, IIS-0911036, and CCF-1214844,
A preliminary version of
this work appeared as ``Ambiguous language and differences in beliefs''
in the \emph{Principles of Knowledge Representation and Reasoning:
Proceedings of the Thirteenth International Conference}, 2012, pp. 329--338.
by AFOSR grants
FA9550-08-1-0438, FA9550-12-1-0040, and FA9550-09-1-0266, and by ARO grant
W911NF-09-1-0281.
The work of Kets was supported in part by AFOSR grant FA9550-08-1-0389.

\bibliographystyle{chicagor}
\bibliography{z,joe}
\end{document}